%% file: main.tex
\documentclass[10pt,twocolumn,letterpaper]{article}
\usepackage{amsfonts}
\usepackage{amsthm}
\usepackage{bm}

\input{tweaks.tex}

\input{macros.tex}
\begin{document}

\title{{\Huge VoronoiNet} \\ {\large General Functional Approximators with Local Support}}

\author{
Francis Williams \\
New York University
\and
Daniele Panozzo \\
New York University
\and
Kwang Moo Yi \\
University of Victoria
\and
Andrea Tagliasacchi \\
Google Brain
}

\maketitle

\begin{abstract}
\ky{Voronoi diagrams are highly compact representations that are used in various Graphics applications.
In this work, we show how to embed a differentiable version of it -- via a novel deep architecture -- into a generative deep network.
By doing so, we achieve a highly compact latent embedding that is able to provide much more detailed reconstructions, both in 2D and 3D, for various shapes. 
}
In this tech report, we introduce our representation and present a set of preliminary results comparing it with recently proposed implicit occupancy networks.
\end{abstract}

\newtheorem{lemma}{Lemma}
\input{1_intro}

\input{3_method}
\input{4_results}
\input{5_future}

\clearpage
{
  \small
  \bibliographystyle{ieee_fullname}
  \bibliography{macros,main}
}
\end{document}

%% file: tweaks.tex
\usepackage{cvpr}
\usepackage{times}
\usepackage{epsfig}
\usepackage{graphicx}
\usepackage{amsmath}
\usepackage{amssymb}

\usepackage[pagebackref=true,breaklinks=true,colorlinks,bookmarks=false]{hyperref}

\cvprfinalcopy %

\def\cvprPaperID{311} %

\ifcvprfinal\fi

\usepackage{overpic}
\usepackage{amssymb}
\usepackage{amsmath}
\usepackage{graphicx}
\usepackage{booktabs} %
\usepackage{multirow} %
\usepackage{diagbox}
\usepackage[T1]{fontenc} %
\usepackage[style=american]{csquotes} %

\usepackage{color}
\definecolor{turquoise}{cmyk}{0.65,0,0.1,0.3}
\definecolor{purple}{rgb}{0.65,0,0.65}
\definecolor{dark_green}{rgb}{0, 0.5, 0}
\definecolor{orange}{rgb}{0.8, 0.6, 0.2}
\definecolor{red}{rgb}{0.8, 0.2, 0.2}
\definecolor{blueish}{rgb}{0.0, 0.7, 1}
\definecolor{light_gray}{rgb}{0.7, 0.7, .7}
\definecolor{pink}{rgb}{1, 0, 1}
\definecolor{dark_red}{rgb}{0.5, 0, 0}

\newcommand{\hide}[1]{{}} %

\newcommand{\TODO}[1]{{\color{red}[TODO: #1]}}

\renewcommand{\TODO}[1]{{}}

\newcommand{\at}[1]{{\color{blue}#1}}
\newcommand{\AT}[1]{{\color{blueish}[AT: #1]}} %
\newcommand{\ky}[1]{{\color{orange}#1}}
\newcommand{\KY}[1]{{\bf \color{orange}[KY: #1]}}

\renewcommand{\at}[1]{{#1}}
\renewcommand{\AT}[1]{}
\renewcommand{\ky}[1]{{#1}}
\renewcommand{\KY}[1]{}

\newcommand{\Figure}[1]{Figure~\ref{fig:#1}}

\newcommand{\eq}[1]{(\ref{eq:#1})}

\usepackage{currfile}

\usepackage{blindtext}

\usepackage{parskip}
\renewcommand{\paragraph}[1]{\vspace{2\parskip}\textbf{#1.~}}

\DeclareMathAlphabet\mathbfcal{OMS}{cmsy}{b}{n}

%% file: macros.tex
\DeclareMathOperator*{\argmin}{arg\,min}
\newcommand{\point}{\mathbf{p}_k}
\newcommand{\points}{\mathbf{P}}
\newcommand{\x}{\mathbf{x}}
\newcommand{\image}{\mathcal{I}}
\newcommand{\cell}{\mathbf{c}}
\newcommand{\R}{\mathbb{R}}
\newcommand{\given}{|}
\newcommand{\object}{\mathcal{O}}
\newcommand{\scalars}{\boldsymbol{\Lambda}}
\newcommand{\scalar}{\boldsymbol{\lambda}_k}
\newcommand{\expect}[2]{\mathbb{E}_{#1}\left[ #2 \right]}

%% file: 1_intro.tex
\section{Introduction}
\label{sec:intro}

\ky{
Choosing a shape representation is a fundamental problem for any geometric task.
Especially, with the advent of deep methods for geometry, it defines what operations are possible (e.g. convolution), what choices of architecture can be used (e.g. graph~\cite{graphcnn} or point networks~\cite{pointnet,pointnet++}), and what input modality (e.g. point clouds or images) can be used for training. 
Naturally, finding a proper differentiable representation for geometry has been of much research interest recently, with much focus on 3D~\cite{occnet,deepsdf,imnet,pifu,pointnet,atlasnet}.
}
A wide variety of 3D representations exist in the literature and are used for a variety of tasks from surface reconstruction~\cite{hoppe,sal,poisson}, shape completion~\cite{scan2mesh}, predicting shape from images~\cite{imnet}, semantic segmentation~\cite{pointnet} and many more.

At a high level, \ky{geometric representations} can be grouped \ky{into two:} \emph{explicit} representations, where the surface of an object is explicitly represented using for example, meshes~\cite{surface-networks}, parameterized patches~\cite{atlasnet, deep-geometri-prior} or point clouds \cite{pointnet, pointnet++}\ky{;} and \emph{implicit} methods, where a 3D object is defined by a scalar function in $\mathbb{R}^3$ (for example by defining the surface as a level set of this function)~\cite{occnet,imnet,pifu,sif, deepsdf, pcnn}.
\ky{With deep networks,} a recent trend is to use a neural network to represent the scalar function for a shape \cite{imnet, occnet, deepsdf, deep-geometri-prior}.
\ky{
Explicit representations have the benefit that they make surface extractions easy -- e.g. via Marching Cubes \cite{marchingcubes} -- while the implicit ones are easy to embed into a deep network with simple architectures.
}
\ky{Recently, \textit{hybrid} representations~\cite{cvxnet,bspnet} have been proposed to combine the best of both.}

\input{fig/teaser.tex}

Of particular relevance to our \ky{work} is CvxNet~\cite{cvxnet}, which represent shapes as the intersection of a finite number of half spaces.
This representation is a universal approximator of convex domains \ky{-- similar to ours --} as well as non-convex ones via composition.
However, \ky{they are still \textit{implicit} when it comes to modelling overlap.
They \textit{train} to make their decompositions non-overlapping through an additional loss term and therefore have no guarantee that it would also be non-overlapping during inference.
While this can be of minor importance for reasoning tasks such as shape classification, it is problematic for others such as physical simulation.
}
\input{fig/interp.tex}

Inspired by \cite{Ahuja:1985}, we propose a novel representation that 
\ky{guarantees non-overlapping convexes.
In other words, any network trained with our representation generates non-overlapping convexes \textit{by construction}.
}
\ky{We encode geometric}
information in the form of a point set~$\points{=}\{\point\}$, and generates the collection of convexes as the corresponding collection of their \textit{Voronoi cells} $\mathbf{C}{=}\{\cell_k\}$.
This representation is hybrid: the position of the seeds is \textit{explicit}, and extracting the surface only requires to compute their \textit{Voronoi Diagram} -- a task for which a number of robust and efficient software libraries exist~\cite{qhull}. 
Note that differently from iso-surface extraction, the Voronoi Diagram is unique  and resolution independent -- no parameter needs to be selected to compute it.
Interestingly for our purposes, it is possible to closely approximate the Voronoi Diagram with a differentiable \textit{implicit} function, which is ideal for training.

%% file: fig/teaser.tex
\begin{figure}
\centering
\begin{overpic} 
[width=\linewidth]
{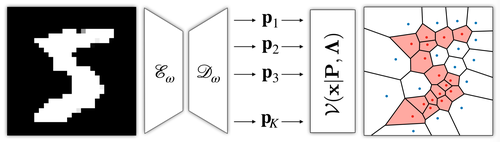}
\end{overpic}
\caption{
We propose a new differentiable \textit{implicit} representation of solid object based on Voronoi diagrams.
An encoder $\mathcal{E}_K$ generate a latent representation, which a decoder $\mathcal{D}_K$ converts into a collection of $K$ sites~$\{\mathbf{p}_k\}$.
Our layer receives these sites in input, and generate a function that can be evaluated at a point $\mathbf{x}$.
}
\label{fig:architecture}
\end{figure}

%% file: fig/interp.tex
\begin{figure*}[ht]
\centering
\begin{overpic} 
[width=\linewidth]
{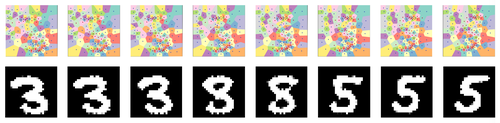}
\end{overpic}
\caption{
We encode the leftmost (3) and rightmost (5) digits in latent space and then linearly interpolate the corresponding latent codes.
}
\label{fig:interp}
\end{figure*}

%% file: 3_method.tex
\section{Method}
\label{sec:method}

We follow the trend pioneered by~\cite{imnet} and seek \emph{functional} networks -- where the output of our network is a \emph{function} that can be queried at a desired location~$\x$.
Given the fixed vector $\scalars{=}(\lambda_k \in \{0, 1\})_k$, we express this function as the the piecewise constant function over the Voronoi diagram of the point set $\points = \{\point \in \mathbb{R}^d\}_k$ where the value of the function at points in the $k^\text{th}$ cell have value $\scalar$:
\begin{equation}
\mathcal{V}(\x \given \scalars, \points) = \scalars 
\left[
\argmin_k \{ \|\x -\point \|_2 \}
\right]
\label{eq:voronoi}
\end{equation}
where we assume that $\Sigma_k \scalar=|\scalars|/2$ -- in other words, we fix half of the sites to represent the ``inside'' (1) of a shape, and other half to represent the ``outside'' (0) of a shape.

Given an input $\image$ (e.g. image, point cloud, voxel grid) from a training dataset $\{\image_n\}$, an encoder $\mathcal{E}_\omega$ maps $\image$ to a latent code $z$ which a decoder $\mathcal{D}_\omega$ maps to the collection of Voronoi centers: $\points {=} \mathcal{D}_\omega(\mathcal{E}_\omega(\image))$. Figure~\ref{fig:architecture} illustrates this architecture visually.
The parameters of encoder and decoder are then trained by minimizing a reconstruction loss:
\begin{equation}\label{eq:rec_loss}
\mathcal{L}_\text{rec}(\omega) = \sum_{n}
\expect{\x \in [0,1]^D}{\| \object_n(\x) - \mathcal{V}(\x \given \mathcal{D}_\omega(\mathcal{E}_\omega(\image_n))) \|_2 }
\end{equation}
where $\object_n$ is the ground truth occupancy function corresponding to $\image_n$.
If we compare our representation to the one provided by ReLU functional networks~\cite{occnet,imnet,deepsdf}, we differ in a fundamental way: our learnable parameters have \textit{localized} support, while the transition boundaries of an MLP generally have a \textit{global} support.

\paragraph{Regularizers}
While the reconstruction loss $\mathcal{L}_\text{rec}$ lies at the core of our method, minimizing this loss is ill-posed. In particular, there exist an infinite space of solutions where voronoi cell agrees with the occupancy of the ground truth.
\at{To remedy this, we develop a number of regularizers that aid our training process.
Notably, these losses do not typically produce pareto-optimal variants of the trained network.}

\begin{lemma}\label{le:ill-posed}
Let $|\points|>6$ be a set of points such that half of $\point \in \mathbb{R}^2$
are labelled 1, and let $\object_n = \mathcal{V}(\x \given \scalars, \points)$
be the occupancy function of the associated Voronoi diagram.
Assume that there are three points labelled 1 so that the triangle they form is
contained in $\object_n$. Then, there exist an infinite number of minimizers
$(\points^*, \scalars^*)$ to \eqref{eq:rec_loss}.
\end{lemma}
\begin{proof}
Assume without loss of generality that $p_1, p_2, p_3$ are all labelled $1$
and the triangle they form is inside $\object_n$. Then let $q$ be any point
inside this triangle. Label $q$ with $1$, and define $(\points^*, \scalars^*)$
by adding this labeled point to $(\points, \scalars)$.
Then $\mathcal{V}(\x \given \scalars^*, \points^*)$ is a minimizer of 
\eqref{eq:rec_loss} for $\object_n$. In fact
$\mathcal{V}(\x \given \scalars^*, \points^*)=\object_n$ since the
$(\points^*, \scalars^*)$ produces the same function as $(\points, \scalars)$.
\end{proof}

\paragraph{Soft-Voronoi}
To differentiate through our Voronoi function, we generalize~\eq{voronoi} by replacing the \textit{argmin} with a \textit{soft-argmin}.
Given $D_k(\x){=}\|\x -\point \|_2$, we first define a vector~$\mathbf{W}$:
\begin{equation}
    \mathbf{W}_k(\x \given \points, \beta) = {e^{-\beta D_k(\x)}}/{\Sigma_k e^{-\beta D_k(\x)}}
\end{equation}
where $\beta {\in} \R^+$ is a temperature parameter and then formulate the soft version of~\eq{voronoi} as:
\begin{equation}
\mathcal{V}(\x \given \scalars, \points, \beta) =
\scalars \cdot \mathbf{W}(\x \given \points, \beta) 
\end{equation}
hence the temperature hyper-parameter $\beta$ controls the soft-argmin approximation to argmin.
In all experiments in the paper we set $\beta{=}10,000$.

\paragraph{Bounds loss} 
We naturally want to prevent our Voronoi sites from drifting far away from the data, which can be enforced in a smooth way via~\cite{bspnet}:
\begin{align}
\mathcal{L}_\text{bound}(\omega) = \sum_{\{\point\}} \sum_{d} \text{soft-bound}(\point[d]) 
\end{align}
where~$[d]$ extracts the $d^\text{th}$ dimension and $\text{soft-bound}(x) =\max({-}x,0) {+} \max(x{-}1,0)$.
We favor this to the use of output layers with bounded ranges as~\cite{bspnet} noting how these can suffer of vanishing gradients.

\input{fig/graphs.tex}
\input{fig/sphere.tex}

\paragraph{Signed distance Loss} 
As we prescribe the Voronoi (inside/outside) classes $\scalars$ rather than optimizing them, it is clear that if the $\lambda_k{=}1$, then the corresponding $\point$ should be \textit{inside}, or in other words $\mathcal{O}(\point){=}1$ (and symmetrically for $\lambda_k{=}0$).
Hence, we can define a loss that induces strong gradients towards the satisfaction of this property. 
With $\phi^+(\x)$ let us define the distance function to $\mathcal{O}$, and with $\phi^-(\x)$ the distance function to its complement space~$\bar{\mathcal{O}}(\x){=}1{-}\mathcal{O}(\x)$, and then define:
\begin{align}
\mathcal{L}_\text{sdf}(\omega) = 
\sum_k \scalar \phi^+(\point) + (1-\scalar) \phi^-(\point)
\end{align}

Note that all correct approximations $\points, \scalars$ of the ground truth occupancy lie in the null space of this loss. Thus, $\mathcal{L}_\text{sdf}$ simply accelerates training and does not prevent the network from finding a global minimum to the problem.  

\paragraph{Centroidal Voronoi loss}
To remedy the ill posedness (Lemma~\ref{le:ill-posed}) of the reconstruction loss \eqref{eq:rec_loss}, we add a loss that pushes each Voronoi point towards the centroid of its corresponding cell. A Voronoi diagram whose points lie at the centroid of its cells is known as \emph{centroidal}. Centroidal Voronoi tesselations have cells with roughly equal shape and have been used for many years in graphics to generate high quality tesselations of space \cite{isotropic,remeshing,voroopt}.
Asking the Voronoi diagram to be as centroidal as possible prevents points from clustering and introduces a unique reconstruction minimum.
Given $m$ Voronoi sites $\points$, we augment the sites with $\sqrt{m}$ points on the boundary with 0 (outside), we compute their Delaunay triangulation, and express its corresponding graph Laplacian operator via a sparse matrix~$\mathbf{L}$; a CVD-like loss can then be expressed by:
\begin{equation}
    \mathcal{L}_\text{cvd}(\omega) = \sum_k \| \mathbf{L} \point \|_2^2
\end{equation}

%% file: fig/graphs.tex
\begin{figure}
\centering
\includegraphics[width=0.49\columnwidth]{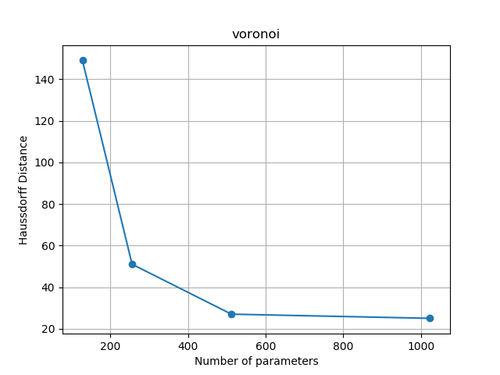}
\includegraphics[width=0.49\columnwidth]{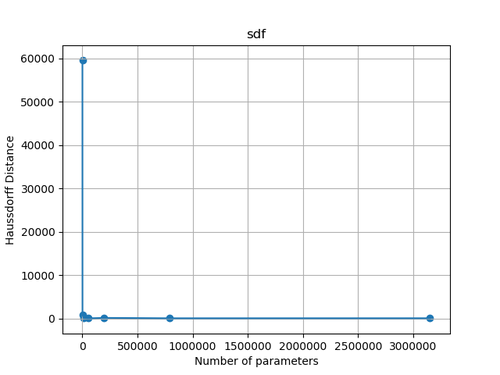}
\caption{Plot of the number of parameters (x-axis) vs. Hausdorff distance (y-axis) from the ground truth for the overfitted sphere example using (left) Voronoi and (right) OccNet. \TODO{DRAFT (and blurry figure)}
}
\label{fig:sphere_graphs}
\end{figure}

%% file: fig/sphere.tex
\begin{figure}
\centering
\begin{overpic} 
[width=\linewidth]
{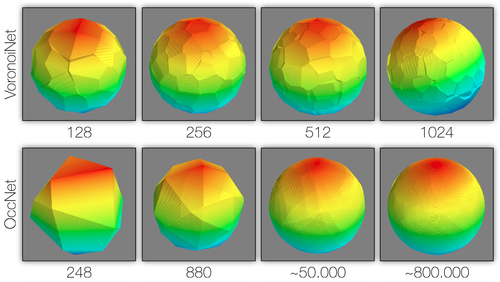}
\end{overpic}
\caption{
We compare the reconstruction power in terms of neural capacity of our VoronoiNet~(top)~vs. the one of traditional multi-layer perceptrons used in OccNet~\cite{occnet}~(bottom) on a simple 3D sphere -- note these are \textit{overfitting} results on a single example.
}
\label{fig:sphere}
\end{figure}

%% file: 4_results.tex
\input{fig/qualitative.tex}
\input{fig/tsne.tex}
\section{Experiments and Results}
\label{sec:results}
\paragraph{Overfitting a Sphere} We start by evaluating the reconstruction power of our network in terms of number of degrees of freedom used for a simple 3D dataset (\Figure{sphere}). We compare our method to the state of the art OccNet \cite{occnet} and DeepSDF \cite{deepsdf}.
Note that while both OccNet and DeepSDF guarantee $C^0$ continuity, the number of neurons necessary to generate reconstructions comparable to Voronoi networks in terms of Hausdorff distance to the ground truth is \textit{three} orders of magnitudes larger than with our approach. Figure~\ref{fig:sphere_graphs} plots the number of parameters for the function versus Haussdorff distance from the ground truth for all 3 methods. Figure~\ref{fig:sphere} shows the reconstructions for each method visually. 

\paragraph{MNIST} We evaluate our formulation on the MNIST dataset by treating the digits as an occupancy function in the~$[0,1]^2$ domain that needs to be predicted. We compare our method against OccNet \cite{occnet}. Both methods use a 4 layer fully connected encoder with 1024 neurons per layer. The encoder maps an MNIST digit image to a 16 dimensional latent variable. The decoder for our method is a 3 layer fully connected network with 1024 neurons per layer which maps the latent code to 128 Voronoi cells. The decoder for occnet has one hidden layer with a varying number of neurons. The decoder maps a latent code and point $\x$ to a probability of occupancy.

\paragraph{Embedding space}
We start by visualizing the tSNE embedding in~\Figure{tsne}.
Notice that while the method was trained in a self-supervised fashion, the latent space was able to organize the various digits by clearly separating the semantic classes.
It is interesting to note how part of the ``8'' embedding space is wedged between the ``3'' and the ``5'', reflecting the geometric similarity between these characters, and the required topological changes to interpolate between them.
To show this, we also visualize a path in the embedding space by encoding two digits, and then interpolating their latent codes; see~\Figure{interp}.
Notice how the topology of the ``9'' is first converted into the one of a ``5'', then into a ``6'' and finally smoothly deformed towards the target configuration.

We conclude our experiments by evaluating (on the test set) the auto-encoding performance on MNIST.
Note in this comparison we keep the capacity of the \textit{encoder} portion of our auto-encoder consistent across the various baselines.
In particular, we compare our Voronoi decoders to popular implicit pipelines that use a multi-layer perceptron as a (conditional) implicit decoder~\cite{occnet,deepsdf,imnet}.
Figure~\ref{fig:occnet_cmp} shows randomly drawn results illustrate how the Voronoi decoder allows for a significantly more compact representation of occupancy than occnet. Table~\ref{tab:occnet_cmp} compares statistics of voronoi reconstructions versus occnet on the test set with varying number of degrees of freedom.

\begin{table}
    \centering
    \begin{tabular}{c c c c}
    \toprule
         \textbf{Method} & \textbf{Mean} & \textbf{Std} & \textbf{Med} \\
    \midrule
         OccNet 128  & 83.803001 & 28.211296 & 85.692169 \\
         OccNet 512  & 76.165771 & 28.211296 & 75.422150 \\
         OccNet 16k  & 52.658348 & 14.332524 & 53.036644 \\
    \midrule
         Voronoi 128 & 57.996124 & 17.018425 & 58.294270 \\
    \bottomrule
    \end{tabular}
    \vspace{1em}
    \caption{Autoencoder statistics for different methods with different degrees of freedom. Note how Voronoi with 128 cells is comparable to OccNet with with 4 orders of magnitude more parameters. }
    \label{tab:occnet_cmp} 
\end{table}

%% file: fig/qualitative.tex
\begin{figure}
\centering
\begin{overpic} 
[width=\linewidth]
{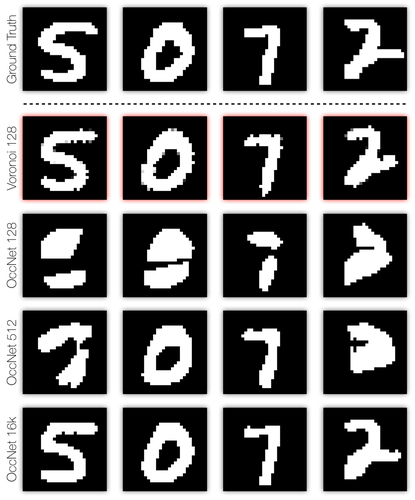}
\end{overpic}
\caption{
A qualitative comparisons of the representation power of different neural decoders as the number of degrees of freedom is increased.
}
\label{fig:occnet_cmp}
\end{figure}

%% file: fig/tsne.tex
\begin{figure*}
\centering
\begin{overpic}
[width=\linewidth]
{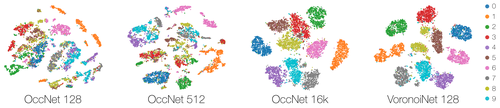}
\end{overpic}
\caption{
A tSNE embedding of our latent code on the MNIST dataset, where the ground truth class has been used to color their identity.
}
\label{fig:tsne}
\end{figure*}

%% file: 5_future.tex
\section{Conclusion \ky{and Future Work}}
\label{sec:future}
We introduced a new differentiable layer for solid geometry representation leveraging the Voronoi diagram.
Similarly to~\cite{occnet,deepsdf,imnet}, we expect our solution to  scale to the modeling of 3D objects with minor modifications. The challenge will be the identification of a random sampling tailored to evaluate the expectation of $\mathcal{L}_\text{rec}$.
While CvxNet~\cite{cvxnet} introduced the idea of hybrid representation learning, where training is performed in the implicit domain, and inference in the explicit domain (i.e.~generates meshes), our network can infer discrete geometry as the crust separating the inside/outside Voronoi cells, removing the need for the iso-surfacing post-processing (e.g.~marching cubes~\cite{marchingcubes}).

\ky{
Our work is early in its stage.
As future work, we plan to apply our method to higher dimensional data, to produce meshing of volume and not only surfaces, to analyze the benefit it brings in physical simulations.
}